\documentclass[american]{llncs} 
\usepackage{todonotes}
\usepackage{babel}
\usetikzlibrary{shapes,shapes.geometric,backgrounds,calc,fit,positioning}

\usepackage[utf8]{inputenc}
\usepackage{graphicx}

\usepackage[ruled,vlined]{algorithm2e}

\usepackage{amsmath,amssymb}
\usepackage{mathtools}
\usepackage{commath}
\usepackage{faktor} 
\usepackage{ntheorem}
\usepackage{nicefrac}
\usepackage{enumerate}
\usepackage{paralist}
\usepackage{mathdesign}
\usepackage{colonequals}

\newcommand*{\coliff}{\ratio\Longleftrightarrow}

\usepackage[hidelinks]{hyperref}
\usepackage[all]{hypcap}

\usepackage[noabbrev]{cleveref}
\let\cref\Cref

\usepackage[backend=biber,style=numeric-comp,giveninits=true,url=false]{biblatex}
\usepackage{booktabs}
\usepackage{arydshln}

\title{Scaling Dimension}
\author{Bernhard Ganter \inst{1} \and Tom Hanika \inst{2,3} \and Johannes Hirth \inst{2,3} }
\date{\today}

\institute{%
 TU Dresden, Dresden, Germany \\ \email{bernhard.ganter@tu-dresden.de} \\[0.5ex]
 \and
 Knowledge \& Data Engineering Group,
 University of Kassel, Germany\\
 \and
 Interdisciplinary Research Center for Information System Design\\
 University of Kassel, Germany\\
 \email{tom.hanika@cs.uni-kassel.de}, \email{hirth@cs.uni-kassel.de}}

\usepackage{fca, newdrawline}

\crefalias{problem}{prob}

\newcommand{\mvc}{\mathbb{D}}
\newcommand{\Scon}{\mathbb{S}}

\DeclareMathOperator{\isd}{ISD}
\DeclareMathOperator{\osd}{OSD}

\newcommand{\x}{$\times$}

\DeclareMathOperator{\Ext}{Ext}

\usepackage{etoolbox,refcount}
\usepackage{multicol}
\newcounter{countitems}
\newcounter{nextitemizecount}
\newcommand{\setupcountitems}{%
  \stepcounter{nextitemizecount}%
  \setcounter{countitems}{0}%
  \preto\item{\stepcounter{countitems}}%
}
\makeatletter
\newcommand{\computecountitems}{%
  \edef\@currentlabel{\number\c@countitems}%
  \label{countitems@\number\numexpr\value{nextitemizecount}-1\relax}%
}
\newcommand{\nextitemizecount}{%
  \getrefnumber{countitems@\number\c@nextitemizecount}%
}
\newcommand{\previtemizecount}{%
  \getrefnumber{countitems@\number\numexpr\value{nextitemizecount}-1\relax}%
}
\makeatother    
    {\end{itemize}%
    \unskip\computecountitems\ifnumcomp{\previtemizecount}{>}{3}{\end{multicols}}{}}
\newcommand\blfootnote[1]{%
  \begingroup
  \renewcommand\thefootnote{}\footnote{#1}%
  \addtocounter{footnote}{-1}%
  \endgroup
}

\begin{document}
\parindent0pt

\maketitle
\blfootnote{Authors are given in alphabetical order.
  No priority in authorship is implied.}

\begin{abstract}
  Conceptual Scaling is a useful standard tool in Formal Concept
  Analysis and beyond. Its mathematical theory, as elaborated in the
  last chapter of the FCA monograph, still has room for
  improvement. As it stands, even some of the basic definitions are in
  flux. Our contribution was triggered by the study of concept
  lattices for tree classifiers and the scaling methods used there. We
  extend some basic notions, give precise mathematical definitions for
  them and introduce the concept of scaling dimension. In addition to a
  detailed discussion of its properties, including an example, we show
  theoretical bounds related to the order dimension of concept
  lattices. We also study special subclasses, such as the ordinal and
  the interordinal scaling dimensions, and show for them first results
  and examples.

\end{abstract}
\begin{keywords}
  Formal~Concept~Analysis, Data~Scaling, Conceptual~Scaling,
  Ferrers~Dimension, Measurement, Preprocessing, Feature~Compression,
  Closed~Pattern~Mining
\end{keywords}

\section{Introduction}
When heterogeneous data needs to be analyzed conceptually, e.g., for
(closed) pattern mining~\cite{interord-gene-scaling}, ontology
learning~\cite{fca-onto-learning} or machine
learning~\cite{tree-views,nn-views}, conceptual scaling~\cite{scaling}
is a tool of choice. The task of this method is to translate given
data into the standard form, that of a formal
context~\cite{fca-book}. There are many ways to do this, but not all
methods are meaningful in every situation. Often the data has an
implicit structure that should guide scaling. For example, it is
natural to analyze ordinal data also ordinally. We propose the notion
of a \emph{pre-scaling} to reveal such implicit assumptions and to
make them  usable for a scaling.

Another important aspect is the complexity of the conceptual structure
created by scaling. Several authors have suggested to restrict to
important attribute
combinations~\cite{smeasure,tree-views,logical-scaling}.  We formalize
this, speaking of conceptual \emph{views} of the data. It turns out
that such views have a natural characterization in terms of
\emph{scale measures}, i.e., continuous maps with respect to closure
systems that are represented by means of formal contexts. This in
turn opens the door to basic theory questions. We address one of them
here for the first time: the question of the \emph{scaling dimension},
i.e., the size of the simplest data set that has the present
conceptual structure as its derivative. We study this for the case of
ordinal and in particular for that of interordinal scaling, proving
characterizations and showing small examples. In addition to our
theoretical findings, we demonstrate the applicability of the scaling
dimension based on the drive concepts data set and provide all used
functions in the \texttt{conexp-clj}~\cite{conexp} tool.

\section{Formal concepts derived from data  tables}

\label{sec:form-conc-conc}

At first glance, it seems very limiting that \FCA focuses on a single
basic data type, that of a binary relation between two sets (a
\emph{formal context}). Data comes in many different formats, so why
restrict to one type only? But this limitation is intentional. It allows a
cleaner separation of objective formal data analysis and subjective
interpretation, and it allows a unified, clear structure of the 
mathematical theory. 

FCA handles the many different data formats in a two-step process.
First, the data is transformed into the standard form -- that is, into
a formal context -- and in the second step that context is
analyzed conceptually. The first step, called \emph{conceptual
  scaling},\footnote{The word ``scaling'' refers to measurement
  theory, not whether algorithms can be applied to large data sets.} 
is understood as an act of interpretation and depends on
subjective decisions of the analyst, who must reveal how the data at
hand is meant. It is therefore neither unambiguous nor automatic, and
usually does not map the data in its full complexity. However, it is
quite possible that several such \emph{conceptual views} together
completely reflect the data. 

In machine learning the classification data that is used for
classifiers are usually lists of n-tuples, as in relational databases
or in so-called data tables. In  
FCA terminology, one speaks of a \emph{many-valued context}. In such
a many-valued context, the rows have different names (thereby forming
a key), and so do the columns, whose names are called
\emph{(many-valued) attributes}. The entries in the table are the
values. 

Formally, a \textbf{many-valued context} $\mvc:=(G,M,W,I)$
consists of a set $G$ of \textbf{objects}, a set $M$ of
\textbf{many-valued attributes}, a set $W$ of \textbf{attribute
  values}, and a ternary relation $I\subseteq G\times M\times W$
satisfying $$(g,m,v)\in I\mbox{ and }(g,m,w)\in I\ \implies v=w.$$
This condition ensures that there is at most one value for each
object-attribute pair. The absence of values is allowed. If no values
are missing, then one speaks of a \textbf{complete} many-valued
context. The value for the object-attribute pair $(g,m)$, if present,
sometimes is denoted by $m(g)$, and $m(g)={\perp}$ indicates that the
value is missing.

It was mentioned above that in order to derive a conceptual structure
from a many-valued context, \FCA requires the interpretative step of
\emph{conceptual scaling}, which determines the concept-forming
attributes which are derived from the many-valued context. That a
formalism is required for this is illustrated by the following
example. Suppose that one of the many-valued attributes is ``size'',
with values ``very small'', ``small'', ``large'', ``very large''.
A simple approach would be to use the values as attribute names,
introducing attributes ``size = very small'', ``size = small'', etc.,
with the obvious interpretation: an object $g$ has the attribute
``size = small'' if $m(g)=\text{small}$ when $m=\text{size}$. In this
\emph{nominal} interpretation the attribute extent for the attribute
``size = small'' contains those objects which are ``small'', but
excludes those which are ``very small'', which perhaps is not
intended. To repair this, one could use the implicit
order $$\text{very small} < \text{small} < \text{large} < \text{very
  large}$$ of the values and dervide attributes such as
$\text{size}\le\text{large}$, etc. But this \emph{interordinal}
interpretation can equally lead to undesired attribute extents, since
the attribute $\text{size}\le \text{large}$ applies to all object
which are ``very small'', ``small'', or ``large'', but excludes those
for which no value was noted because their size was unremarkable,
neither small nor large. A \emph{biordinal} interpretation can take
this into account~\cite{fca-book}.

\subsection{Pre-scalings}
\label{Pre-scaling}

Some data tables come with slightly richer information, for which we
introduce an additional definition. A \textbf{pre-scaling} of a
many-valued context $\mvc\coloneqq(G,M,W,I)$ is a family $(W(m)\mid m\in M)$ of
sets $W(m)\subseteq W$ such that $W=\bigcup_{m\in M} W(m)$
and $$(g,m,w)\in I \implies w\in W(m)$$ for all $g\in G, m\in M$. We
call $W(m)$ the \textbf{value domain} of the many-valued attribute
$m$.  A tuple $(v_{m}\mid m\in M)$ \textbf{matches} a
pre-scaling iff $v_{m}\in W(m)\cup\{\perp\}$ holds for all $m\in M$.
$(G,M,W(m)_{m\in M},J)$ may be called a \textbf{stratified}
many-valued context. 

It is also allowed that the value domains additionally carry a
structure,  e.g., are ordered. This also falls under the definition of
``pre-scaling''. We remain a little vague here, because its seems
premature to give a sharp definition. Prediger~\cite{Prediger96}
suggests the notion of a \emph{relational} many-valued
context. This may be formalized as a
tuple $$(G,M,(W(m),\mathcal{R}_{m})_{m\in M}, I),$$
where $(G,M,W(m)_{m\in M},I)$ is a stratified many valued-context as defined
above, where on each value domain $W(m)$ a family $\mathcal{R}_{m}$ of
relations is given. Prediger and Stumme~\cite{logical-scaling} then
discuss deriving one-valued attributes using expressions in a suitable
logical language, such as one of the OWL-variants. They call this
\emph{logical scaling}. 
\subsection{Interordinal plain scaling}
\label{sec:plain-scaling}

A \textbf{scale} for an attribute $m$ of a many-valued context
$(G,M,W(m)_{m\in M},I)$ is a formal context
$\mathbb{S}_{m}:=(G_{m},M_{m},I_{m})$ with $W(m)\subseteq G_{m}$. The
objects of a scale are the \textbf{scale values}, the attributes are
called \textbf{scale attributes}.

By specifying a scale a data analyst determines how the attribute
values are used conceptually.

For \textbf{plain scaling} a formal context $(G,N,J)$ is
\textbf{derived} from the many-valued context $(G,M,W(m)_{m\in M},I)$
and the scale contexts $\mathbb{S}_{m}$, $m\in M$, as follows:
\begin{itemize}
\item The object set is $G$, the same as for the many-valued context,
\item the attribute set is the disjoint union of the scale attribute
  sets, formally $$N:=\bigcup_{m\in M}\{m\}\times M_{m},$$ 
\item and the incidence is given by $$g\,J\,(m,n) \coliff (g,m,v)\in I\mbox{
  and } v\,I_{m}\,n.$$
\end{itemize}

The above definition may look technical, but what is described is
rather simple: Every column of the data table is replaced by several
columns, one for each scale attribute of $\mathbb{S}_{m}$, and if the
cell for object $g$ and many-valued attribute $m$ contains the value
$v$, then that is replaced by the corresponding ``row'' of the
scale. Choosing the scales is already an act of interpretation, deriving the
formal context when scales are given is deterministic.

\emph{Pre-scaling}, as mentioned above, may suggest the scales to
use. An ordered pre-scaling naturally leads to an \emph{interordinal}
interpretation of data, using only interordinal scales. We repeat the
standard definition of interordinal scaling:

\begin{definition}[Interordinal Scaling of $\mvc$]
  \label{def:interordinal}
  When $\mvc\coloneqq (G,M,W,I)$ is a many-valued context with
  linearly ordered value sets $(W(m),\leq_{m})$, then the formal
  context $\mathbb{I}(\mvc)$ \textbf{derived from interordinal
    scaling} has $G$ as its object set and attributes of the
  form $$(m,\leq_{m}, v)\mbox{ or }(m,\geq_{m}, v),$$ where $v$ is a
  value of the many valued attribute $m$. The incidence is the obvious
  one, an object $g$ has e.g., the attribute $(m,\leq_{m},v)$ iff the
  value of $m$ for the object $g$ is $\leq_{m} v$. Instead of
  $(m,\leq_{m},v)$ or $(m,\geq_{m},v)$ one writes $m{\,:\;}\leq v$ and
  $m{\,:\;}\geq v$, respectively. Formally $\mathbb{I}(\mvc) \coloneqq
  (G,N,J)$, where
  \[N:=\{m{\,:\;}\leq v\mid m\in M,v\in W(m)\}
  \cup\{m{\,:\;}\geq v\mid m\in M,v\in W(m)\}\] and
  \[ (g, m{\,:\;}\leq v)\in J\coliff m(g)\leq v,\qquad(g, m{\,:\;}\geq v)\in J\coliff m(g)\geq v.\]
 For simplicity, attributes which apply to all or to no objects are usually omitted.
\end{definition}

\noindent\textbf{Remark: } The formal context derived from a
many-valued context $\context[D]$ with linearly ordered value sets via
\emph{ordinal plain scaling} is denoted $\mathbb{O}(\context[D])$.

\subsection{Scale measures and views}
\label{sec:scale-measures-views}

By definition, the number of attributes of a derived context is the
sum of the numbers of attributes of the scales used, and thus tends to
be large. It is therefore common to use selected subsets of these
derived attributes, or attribute combinations. This leads to the
notion of a \emph{view}:
\begin{definition}
  A \textbf{view} of a formal context $(G,M,I)$ is a formal
  context $(G,N,J)$, where for each $n\in N$ there is a set
  $A_{n}\subseteq M$ such that
  $$g\,J\,n \coliff A_{n}\subseteq g^{I}.$$
  A \textbf{contextual view} of a many-valued context $\context$ is a
  view of a  derived context of $\context$; the concept lattice of
  such a contextual view is a \textbf{conceptual view} of $\context$.   
\end{definition}

In order to compare contexts derived by conceptual scaling, the
notion of a \emph{scale measure} is introduced. 
\begin{definition}
Let $\context:=(G,M,I)$ and $\ \context[S]:=(G_{\context[S]},
M_{\context[S]},I_{\context[S]})$ be formal contexts.
A mapping $$\sigma:G\to G_{\context[S]}$$ is called an 
$\context[S]$-\textbf{measure} of $\context$ if the preimage
$\sigma^{-1}(E)$ of every extent $E$ of $\ \context[S]$ is an
extent of $\context$. An $\;\context[S]$-measure is \textbf{full} if
every extent of $\;\context$ is the preimage of an extent of $\;\context[S]$.  
\end{definition}

\begin{proposition}\label{prop:identity-measure}
  A formal context $\;\context[K]_{1}:=(G,N,J)$ is a view of
  $\;\context:=(G,M,I)$ if and only if the identity map is a
  $\;\context[K]_{1}$-measure of $\;\context$.
\end{proposition}
\begin{proof}
  When $(G,N,J)$ is a view of $(G,M,I)$, then every extent $E$ of
  $(G,N,J)$ is of the form $E=B^{J}$ for some $B\subseteq N$. Then $E$
  also is an extent of $(G,M,I)$, since $E=(\bigcup_{n\in B}
  A_{n})^{I}$. Conversely, if the 
  identity map is a $(G,N,J)$-measure of $(G,M,I)$, then for each
  $n\in N$ the preimage of its attribute extent $n^{J}$ (which, of
  course, is equal to $n^{J}$) must be an extent of $(G,M,I)$ and
  therefore be of the form $A^{I}_{n}$ for some set $A_{n}\subseteq M$.
\qed
\end{proof}

As shown by \cref{prop:identity-measure} there is a close tie between
contextual views and canonical representation of scale-measures as
proposed in Proposition 10 of~\textcite{smeasure}. Said
representations provide for a scale-measure $\sigma$ of $\context$
into $\Scon$ an equivalent scale-measure based on the identity map of
$\context$ into a context of the form $(G,\mathcal{A},\in)$ where
$\mathcal{A}\subseteq \Ext(\context)$. With the now introduced notions
we understand the context $(G,\mathcal{A},\in)$ as
contextual view of $\context$.

\section{Measurability}\label{sec:measurability}
A basic task in the theory of conceptual scaling is to decide if a
given formal context $\context$ is derived from plain scaling (up to
isomorphism). More precisely, one would like to decide whether
$\context$ can be derived using a set of given scales, e.g., from
interordinal scaling.

A key result here is Proposition~122 of the FCA book
\cite{fca-book}. It answers the following question: Given a formal
context $\context$ and a family of scales $\mathcal{S}$, does there
exist some many-valued context $\context[D]$ such that $\context$ is
(isomorphic to) the context derived from plain scaling of
$\context[D]$ using only scales from $\mathcal{S}$?  The proposition
states that this is the case if and only if $\context$ is fully
$\mathcal{S}$-measurable (cf. Definition 94 in~\textcite{fca-book}),
i.e., fully measurable into the \emph{semiproduct} of the scales in
$\mathcal{S}$.

Based on this proposition, Theorem~55 of the FCA book also gives some
simple characterizations for measurability, one of them concerning
interordinal scaling:

\begin{theorem}[Theorem 55 \cite{fca-book}]\label{th55}
  A finite formal context is derivable from interordinal scaling iff
  it is atomistic (i.e.,
  $\{g\}'\subseteq \{h\}'\implies \{g\}'=\{h\}'$) and the complement
  of every attribute extent is an extent.
\end{theorem}

While this characterization may seem very restrictive there are
potential applications machine learning. In \textcite{tree-views} the
authors discuss how Decision Trees or Random Forest classifiers can be
conceptually understood. Similar investigations have been made for
latent representations of neural networks\cite{nn-views}.  The nature
of the decision process suggests an interordinal scaling on the data
set. They study several conceptual views based on interordinal
scalings, asking if they can be used to explain and interprete tree
based classifiers.  One of these possibilities they call the
\textbf{interordinal predicate view} of a set of objects $G$ with
respect to a decision tree classifier $\mathcal{T}$. We refer
to~\textcite{tree-views} for a formal definition and a thorough
investigation of how to analyze ensembles of decision tree classifiers
via views.

\begin{proposition}
  The interordinal predicate view $\mathbb{I}_{\mathcal{P}}(G,
  \mathcal{T})$ is derivable from interordinal scaling. 
\end{proposition}
\begin{proof}
  The interordinal predicate view context is atomistic, as shown
  in Proposition~1 of \textcite{tree-views}. For every attribute
  $P\in\mathcal{T}(\mathcal{P})$ we find that
  $\{P\}'\in\Ext(\mathbb{I}_{\mathcal{P}}(G, \mathcal{T}))$. We have
  to show that $G\setminus \{P\}'\in \Ext(\mathbb{I}_{\mathcal{P}}(G,
  \mathcal{T}))$. This is true since we know that for any predicate
  $P$ that $\{P\}'=\{g\in G\mid g\models P\}$ by definition and
  therefore $G\setminus\{P\}' \overbrace{=}^{\star} \{g\in G\mid
  g\not\models P\}=\{\neg P\}'$, which is an extent. The equality
  $(\star)$ follows directly from the fact that the many valued
  context $\mvc$ is complete. Thus, by~\cref{th55} the proposition holds.
\qed
\end{proof}

Other characterizations (also in Theorem 55 of the FCA book \cite{fca-book}) show
that while every context is fully ordinally measurable, a context is
fully nominally measurable iff it is atomistic. Thus every fully
interordinally measurable context is also fully nominally
measurable\label{th55}. In addition to that, the scale families of
contranominal, dichtomic and interordinal scales are equally
expressive.

While equally expressive, a natural quantity in which these families
and scaling in general differs in how complex the scaling is in terms
of the size the many-valued context $\mvc$. This is expressed by the
\emph{scaling dimension}, to be introduced in the following definition.

\begin{definition}[Scaling Dimension]\label{def:sd}
  Let $\context\coloneqq(G,M,I)$ be a formal context and let
  $\mathcal{S}$ be a family of scales. The \textbf{scaling dimension}
  of $\context$ with respect to $\mathcal{S}$ is the smallest number
  $d$ such that there exists a many-valued context\linebreak
  $\mvc\coloneqq(G,M_{\mathbb{D}},W_{\mathbb{D}},I_{\mathbb{D}})$
  with $|M_{\mathbb{D}}|=k$, such that $\context$ has the same extents
  as the
  context derived from $\mvc$ when only scales from $\mathcal{S}$ are
  used. If no such scaling exists, the dimension remains undefined.
\end{definition}

The so-defined dimension can also be perceived as an instance of the
feature compression problem. Even when it is known that $\context$ can
be derived from a particular data table, it may be that there is
another, much simpler table from which one can also derive $\context$
(cf. \cref{fig:interordinal-dimension} for an example).

To prove properties of the scaling dimension, we need results about
scale measures which were published
elsewhere~\cite{fca-book,smeasure}. The following lemma follows from
Propositions 120 and 122 from \textcite{fca-book}.

\begin{lemma}\label{lem:interordinal-measurement} Let
  $\mvc$ be a complete many-valued context. For every many-valued
  attribute $m$ of $\;\mvc$ let
  $\;\mathbb{S}_{m}:=(G_{m},M_{m},I_{m})$   be a scale for attribute $m$,  
  i.e., with $W(m)\subseteq G_{m}$.
  Furthermore, let $\context$ be the formal context derived from
  $\mvc$ through plain scaling with the scales $(\mathbb{S}_{m}\mid{m\in
    M})$. Recall that $\context$ and $\mvc$ have the same object set
  $G$, and that for every $m\in M$ the mapping  $\sigma_{m}: G\to
  W(m)$,  defined by $g\mapsto m(g)$, is a scale measure from $\context$ to
  $\mathbb{S}_{m}$.

  Then the extents of $\,\context$ are exactly the intersections of
  preimages of extents of the scales $\mathbb{S}_{m}$. 
\end{lemma}

A first result on the scaling dimension is easily obtained for the
case of ordinal scaling. It was already mentioned that every formal
context is fully ordinally measurable, which means that every context
is (up to isomorphism) derivable from a many-valued context $\mvc$ through
plain ordinal scaling. But how large must this $\mvc$ be, how many
many-valued attributes are needed? The next proposition gives the
answer. For simplicity, we restrict to the finite case.
\begin{proposition}\label{prop:ordinal-dimension}
The ordinal scaling dimension of a finite formal context $\context$ equals
the width of the ordered set of infimum-irreducible concepts.  
\end{proposition}
\begin{proof} 
The width is equal to the smallest number of chains $\mathcal{C}_{i}$
covering the\linebreak $\subseteq$-ordered set of irreducible attribute extents.
 From these chains we can construct a many-valued context $\mvc$ with
 one many-valued attribute $m_{i}$ per chain $\mathcal{C}_i$. The
 values of $m_{i}$ are the elements of the chain $\mathcal{C}_{i}$,
 and the order of that chain is understood as an ordinal prescaling.  
 The derived context by means of ordinal scaling has exactly the set
 of all intersections of chain  extents as extents (Proposition 120
 \cite{fca-book}), i.e., the   set of all $\bigcap \mathcal{A}$ where
    $\mathcal{A}\subseteq \mathcal{C}_1\times \dotsb\times
    \mathcal{C}_w$. Those are exactly the extents of $\context$. This
    implies that the scaling dimension is less or equal to the width.

 But the converse inequality holds as well.  Suppose $\context$ has
 ordinal scaling dimension $w$. Then by
 Lemmaref{lem:interordinal-measurement} every extent of $\context$ is
 the intersection of preimages of extents of the individual scales.
 For $\cap$-irreducible extents  this means that they must each be a
 preimage of an extent from one of the scales. Incomparable extents
 cannot come from the same (ordinal) scale, and 
 thus the scaling must use at least $w$ many ordinal scales.\qed 
\end{proof}

As a proposition we obtain that the ordinal scaling dimension must be at
least as large as the order dimension:
\begin{proposition}[Ordinal Scaling Dimension and Order
  Dimension]\label{odim-osd} 
  The order dimension of the concept lattice $\BV(\context)$ is a
  lower bound for the ordinal scaling dimension of $\context$.
\end{proposition}
\begin{proof}
  It is well known that the order dimension of $\BV(\context)$ equals
  the Ferrers dimension of $\context$, which remains the same when
  $\context$ is the standard context. The Ferrers relation is the
  smallest number of staircase-shaped relations to fill the
  complement of the incidence relation of $\context$.

  For a context with ordinal scaling dimension equal to $w$ we can
  conclude that the 
  (irreducible) attributes can be partitioned into $w$ parts, one for
  each chain, such that for each part the incidence is
  staircase-shaped, and so are the non-incidences. Thus we can derive
  $w$ Ferrers relations to fill all non-incidences.\qed 
\end{proof}

A simple example that order dimension and ordinal scaling dimension
are not necessarily the same is provided by any context with the
following incidence:
\begin{center} {\scriptsize
    $\begin{array}{|c|c|c|}\hline \times&&\\\hline&\times&\\
       \hline&&\times\\\hline\end{array}$}
\end{center}
Its Ferrers dimension is two, but there are three pairwise
incomparable irreducible attributes, which forces its ordinal scaling
dimension to be three.

 A more challenging problem is to determine the interordinal scaling
 dimension of a context $\context$. We investigate this with the help
 of the following definition.


%
\begin{definition}
An \textbf{extent ladder} of $\context$ is a set $\mathcal{R}\subseteq
\Ext(\context)$ of nonempty extents that satisfies:
\begin{enumerate}[i)]
\item the ordered set $(\mathcal{R},\subseteq)$ has width $\leq 2$,
  i.e., $\mathcal{R}$ does not contain three mutually incomparable
  extents, and
\item $\mathcal{R}$ is closed under complementation, i.e., when
  $A\in \mathcal{R}$, then also $G\setminus A \in\mathcal{R}$.
\end{enumerate}
\end{definition}
Note that a finite (and nonempty) extent ladder is the disjoint union
of two chains of equal cardinality, for the following reason: 
Consider a minimal extent $E$ in the ladder. Any other extent must
either contain $E$ or be contained in the complement of $E$, because
otherwise there would be three incomparable extents. The extents
containing $E$ must form a chain, and so do their complements,
which are all contained in the complement of $E$.

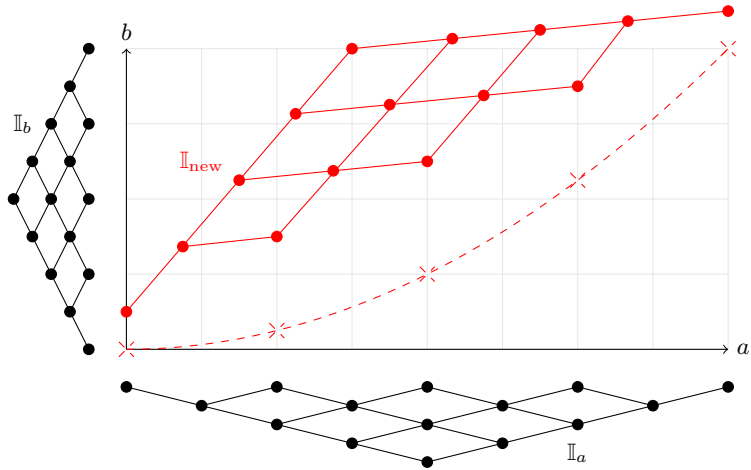
\begin{figure}
  \centering
  \begin{tikzpicture}[domain=0:8]
    \draw[very thin,color=gray,opacity=0.2] (0,0) grid (8,4);
    
    \draw[->] (0,0) -- (8,0) node[right] {$a$};
    \draw[->] (0,0) -- (0,4) node[above] {$b$};
    
    \draw[color=red,dashed,smooth] plot[mark=x,smooth,mark indices={1,7,...,35},mark size=4pt,thick] ({\x},{0.25* \x * 0.25* \x}) node[right] {};    
    \draw[] (-0.5, 0) -- (-1.5,2);
    \draw[] (-1.5, 2) -- (-0.5,4);

    \draw[] (-0.5, 1) -- (-0.75,0.5);
    \draw[] (-0.5, 2) -- (-1,1);
    \draw[] (-0.5, 3) -- (-1.25,1.5);

    \draw[] (-0.5, 1) -- (-1.25,2.5);
    \draw[] (-0.5, 2) -- (-1,3);
    \draw[] (-0.5, 3) -- (-0.75,3.5);
    
    \node[circle,draw,fill=black,minimum size = 4pt,inner sep=0pt] () at (-0.5,0){};
    \node[circle,draw,fill=black,minimum size = 4pt,inner sep=0pt] () at (-0.5,1){};
    \node[circle,draw,fill=black,minimum size = 4pt,inner sep=0pt] () at (-0.5,2){};
    \node[circle,draw,fill=black,minimum size = 4pt,inner sep=0pt] () at (-0.5,3){};
    \node[circle,draw,fill=black,minimum size = 4pt,inner sep=0pt] () at (-0.5,4){};

    \node[circle,draw,fill=black,minimum size = 4pt,inner sep=0pt] () at (-0.75,0.5){};
    \node[circle,draw,fill=black,minimum size = 4pt,inner sep=0pt] () at (-0.75,1.5){};
    \node[circle,draw,fill=black,minimum size = 4pt,inner sep=0pt] () at (-0.75,2.5){};
    \node[circle,draw,fill=black,minimum size = 4pt,inner sep=0pt] () at (-0.75,3.5){};

    \node[circle,draw,fill=black,minimum size = 4pt,inner sep=0pt] () at (-1,1){};
    \node[circle,draw,fill=black,minimum size = 4pt,inner sep=0pt] () at (-1,2){};
    \node[circle,draw,fill=black,minimum size = 4pt,inner sep=0pt] () at (-1,3){};

    \node[circle,draw,fill=black,minimum size = 4pt,inner sep=0pt] () at (-1.25,1.5){};
    \node[circle,draw,fill=black,minimum size = 4pt,inner sep=0pt] () at (-1.25,2.5){};

    \node[circle,draw,fill=black,minimum size = 4pt,inner sep=0pt] () at (-1.5,2){};
    \node[text=black] at (-1.375,3) {$\mathbb{I}_{b}$};
    \draw[] (0, -0.5) -- (4,-1.5);    
    \draw[] (4, -1.5) -- (8,-0.5);    
    
    \draw[] (2, -0.5) -- (5,-1.25);    
    \draw[] (4, -0.5) -- (6,-1);    
    \draw[] (6, -0.5) -- (7,-0.75);    

    \draw[] (2, -0.5) -- (1,-0.75);    
    \draw[] (4, -0.5) -- (2,-1);    
    \draw[] (6, -0.5) -- (3,-1.25);    

    \node[circle,draw,fill=black,minimum size = 4pt,inner sep=0pt] () at (0,-0.5){};
    \node[circle,draw,fill=black,minimum size = 4pt,inner sep=0pt] () at (2,-0.5){};
    \node[circle,draw,fill=black,minimum size = 4pt,inner sep=0pt] () at (4,-0.5){};
    \node[circle,draw,fill=black,minimum size = 4pt,inner sep=0pt] () at (6,-0.5){};
    \node[circle,draw,fill=black,minimum size = 4pt,inner sep=0pt] () at (8,-0.5){};

    \node[circle,draw,fill=black,minimum size = 4pt,inner sep=0pt] () at (1,-0.75){};
    \node[circle,draw,fill=black,minimum size = 4pt,inner sep=0pt] () at (3,-0.75){};
    \node[circle,draw,fill=black,minimum size = 4pt,inner sep=0pt] () at (5,-0.75){};
    \node[circle,draw,fill=black,minimum size = 4pt,inner sep=0pt] () at (7,-0.75){};

    \node[circle,draw,fill=black,minimum size = 4pt,inner sep=0pt] () at (2,-1){};
    \node[circle,draw,fill=black,minimum size = 4pt,inner sep=0pt] () at (4,-1){};
    \node[circle,draw,fill=black,minimum size = 4pt,inner sep=0pt] () at (6,-1){};

    \node[circle,draw,fill=black,minimum size = 4pt,inner sep=0pt] () at (3,-1.25){};
    \node[circle,draw,fill=black,minimum size = 4pt,inner sep=0pt] () at (5,-1.25){};

    \node[circle,draw,fill=black,minimum size = 4pt,inner sep=0pt] () at (4,-1.5){};
    \node[text=black] at (6,-1.375) {$\mathbb{I}_{a}$};


    \draw[draw=red] (0, 0.5) -- (3,4);    
    \draw[draw=red] (3, 4) -- (8,4.5);    

    \draw[draw=red] (2, 1.5) -- (4.33333,4.13333);  
    \draw[draw=red] (4, 2.5) -- (5.5,4.25);  
    \draw[draw=red] (6, 3.5) -- (6.666666,4.366666);  

    \draw[draw=red] (2, 1.5) -- (0.75,1.36666);  
    \draw[draw=red] (4, 2.5) -- (1.5,2.25);  
    \draw[draw=red] (6, 3.5) -- (2.25,3.13333);

    \node[circle,draw=red,fill=red,minimum size = 4pt,inner sep=0pt] () at (0,0.5){};
    \node[circle,draw=red,fill=red,minimum size = 4pt,inner sep=0pt] () at (2,1.5){};
    \node[circle,draw=red,fill=red,minimum size = 4pt,inner sep=0pt] () at (4,2.5){};
    \node[circle,draw=red,fill=red,minimum size = 4pt,inner sep=0pt] () at (6,3.5){};%
    \node[circle,draw=red,fill=red,minimum size = 4pt,inner sep=0pt] () at (8,4.5){};

    \node[circle,draw=red,fill=red,minimum size = 4pt,inner sep=0pt] () at (3,4){};
    \node[circle,draw=red,fill=red,minimum size = 4pt,inner sep=0pt] () at (4.33333,4.1333){};
    \node[circle,draw=red,fill=red,minimum size = 4pt,inner sep=0pt] () at (5.5,4.25){};
    \node[circle,draw=red,fill=red,minimum size = 4pt,inner sep=0pt] () at (6.666666,4.36666){};

    \node[circle,draw=red,fill=red,minimum size = 4pt,inner sep=0pt] () at (0.75,1.36666){};
    \node[circle,draw=red,fill=red,minimum size = 4pt,inner sep=0pt] () at (1.5,2.25){};
    \node[circle,draw=red,fill=red,minimum size = 4pt,inner sep=0pt] () at (2.25,3.133333){};%

    \node[circle,draw=red,fill=red,minimum size = 4pt,inner sep=0pt] () at (2.75,2.375){};

    \node[circle,draw=red,fill=red,minimum size = 4pt,inner sep=0pt] () at (3.5,3.2555533333333333){};
    \node[circle,draw=red,fill=red,minimum size = 4pt,inner sep=0pt] () at (4.75,3.3777766666666666){};

    \node[text=red] at (1,2.5) {$\mathbb{I}_{\text{new}}$};
  \end{tikzpicture}
  \caption{This figure displays the $a$ and $b$ feature of five data
    points and their respective interordinal scales $\mathbb{I}_{a}$
    and $\mathbb{I}_{b}$ (black). The interordinal scaling dimension
    of this data set is one and the respective reduced interodinal
    scale $\mathbb{I}_{\text{new}}$ is depicted in red. The reduction
    would then remove the $a$ and $b$ data column and substitute it
    for a new column given by $\mathbb{I}_{\text{new}}$.}
  \label{fig:interordinal-dimension}
\end{figure}

\begin{theorem}[Interordinal Scaling Dimension]\label{isd-bound}
  The interordinal scaling dimension of a finite formal context
  $\;\context$, if it exists, is equal to the smallest number of extent
  ladders, the union of which contains all meet-irreducible 
  extents of $\context$.
\end{theorem}
\begin{proof}
Let $\context$  be a formal context with interordinal scaling
dimension $d$. W.l.o.g.\ we may assume that $\context$ was derived by
plain interordinal scaling from a many-valued context $\mvc$ with $d$
many-valued attributes. We have to show that the irreducible attribute
extents of $\context$ can be covered by $d$ extent ladders, but not by
fewer.

To show that $d$ extent ladders suffice, note that the extents of an
interordinal scale form a ladder, and so do their preimages under a
scale measure. Thus Lemma~\ref{lem:interordinal-measurement} provides
an extent ladder for each of the $d$ scales, and every extent is an
intersection of those. Meet-irreducible extents cannot be obtained
from a proper intersection and therefore must all be contained in one of
these ladders.

For the converse assume that $\context$ contains $l$ ladders covering
all meet-irreducible extents. From each such ladder $\mathcal{R}_{i}$
we define a formal context $\context[R]_{i}$, the attribute extents of
which are precisely the extents of that ladder, and note that this
context is an interordinal scale (up to clarification). Define a
many-valued context with $l$ many-valued attributes $m_{i}$. The
attribute values of $m_{i}$ are the minimal non-empty intersections of
ladder extents, and the incidence is declared by the rule that
  an object $g$ has the value $V$ for the attribute $m_{i}$ if $g\in V$.
The formal context derived from this many-valued context by plain
interordinal scaling with the scales $\context[R]_{i}$ has the same
meet-irreducible extents as $\context$, and therefore the same
interordinal scaling dimension. Thus $l\ge d$.
\qed\end{proof}

\begin{proposition}
  Let $w$ denote the width of the ordered set of meet-irreducible
  extents of the formal context $\context$. The interordinal scaling
  dimension of $\context$, if defined, is bounded below by
  $\nicefrac{w}{2}$ and bounded above by $w$.
\end{proposition}
\begin{proof}
An extent ladder consists of two chains, and $w$ is the smallest
number of chains covering the meet-irreducible extents. So at least
$\nicefrac{w}{2}$ ladders are required.

Conversely from any covering of the irreducible extents by $w$ chains
a family of $w$ ladders is obtained by taking each of these chains
together with the complements of its extents.
\qed\end{proof}



A context where $\osd(\context)\neq 2\cdot\isd(\context)$ is depicted in
the next section in \cref{fig:drive-concepts}.

Another inequality that can be found in terms of many-valued contexts.
For a many-valued context $\mathbb{D}$ and its ordinal scaled context
$\mathbb{O}(\mathbb{D})$ and interordinal scaled context
$\mathbb{I}(\mathbb{D})$ is the ISD of $\mathbb{I}(\mathbb{D})$ in
general not equal to the OSD of $\mathbb{O}(\mathbb{D})$. Consider for
this the counter example given in \cref{fig:unequal-osd-isd}. The
depicted many-valued context has two ordinally pre-scaled attributes
that form equivalent interordinal scales.

\begin{figure}
  \centering  
  \begin{tabular}{|l||c|c|}
    \hline 
    $\mathbb{D}$&$m_1$&$m_2$\\
    \hline\hline
    $g_1$&1&d\\
    $g_2$&2&c\\
    $g_3$&3&b\\
    $g_4$&4&a\\ \hline
  \end{tabular}
  \caption{Example many-valued context where the attribute values are
    ordinally pre-scaled by $1<2<3<4$ and $a<b<c<d$. The interordinal
    scaling dimension of  $\mathbb{I}(\mathbb{D})$ is one and the
    ordinal scaling dimension of $\mathbb{O}(\mathbb{D})$ is two.}
  \label{fig:unequal-osd-isd}
\end{figure}


\section{Small Case Study}
\begin{figure}[b]
  \centering
    \begin{tabular}{|l||c|c|c|c|c|c|}
    \hline 
      &1&2&3&4&5&6\\
      \hline\hline
      Conventional&&&\x&\x&\x&\x\\\hline
      All-Wheel&\x&&&\x&\x&\x\\\hline
      Mid-Wheel&\x&\x&&&\x&\x\\\hline
      Rear-Wheel&\x&\x&\x&&&\x\\\hline
      Front-Wheel&\x&\x&\x&\x&\x&\\
      \hline
    \end{tabular}
  \label{fig:drive-concepts-ctx}
  \caption{The standard context of the drive concepts lattice,
    cf. Figure 1.13 in \textcite{fca-book}.}
\end{figure}

\begin{figure}[ht]
  \centering
  \colorlet{mivertexcolor}{black!80}
\colorlet{jivertexcolor}{black!80}
\colorlet{vertexcolor}{black!80}
\colorlet{bordercolor}{black!80}
\colorlet{leddar1color}{red!80}
\colorlet{leddar2color}{blue!50}
\colorlet{leddar3color}{orange!80}
\colorlet{linecolor}{gray}
\tikzset{vertexbase/.style 2 args={semithick, shape=circle, inner sep=2pt, outer sep=0pt, draw=bordercolor},%
  vertex/.style 2 args={vertexbase={#1}{}, fill=vertexcolor!45},%
  leddar1/.style 2 args={vertexbase={#1}{}, fill=leddar1color},%
  leddar2/.style 2 args={vertexbase={#1}{}, fill=leddar2color},%
  leddar3/.style 2 args={vertexbase={#1}{}, fill=leddar3color},%
  conn/.style={-, thick, color=linecolor}%
}
\tikzstyle{o} = [text width=0.8cm,font=\small\linespread{-1}\selectfont]
\tikzstyle{c} = [text width=2cm,align=center]
\tikzstyle{r} = [text width=2.6cm,align=right]
\tikzstyle{l} = [text width=2cm,align=left]
\begin{tikzpicture}[scale=0.8,font=\footnotesize]
  \begin{scope} 
    \begin{scope} 
      \foreach \nodename/\nodetype/\param/\xpos/\ypos in {%
        0/vertex//6.5040822140822065/-1.6418844118844103,
        1/leddar1//2.1452258852258783/-0.10192918192918121,
        2/vertex//4.846672771672765/-0.10192918192918121,
        3/vertex//8.148441188441181/-0.036676841676840155,
        4/leddar2//6.491031746031739/0.028575498575499125,
        5/leddar1//11.619865689865682/0.10687830687830768,
        6/leddar3//6.595435490435483/1.7903886853886863,
        7/vertex//8.070138380138372/1.8425905575905581,
        8/vertex//4.846672771672765/1.8556410256410256,
        9/vertex//2.0538726088726023/1.9469943019943035,
        10/leddar1//-0.4387667887667952/1.960044770044771,
        11/leddar1//13.929798534798527/2.116650386650388,
        12/vertex//11.75037037037036/2.16885225885226,
        13/vertex//6.595435490435483/3.4608485958485957,
        14/leddar1//14.03420227920227/4.022018722018723,
        15/leddar1//-0.582321937321943/4.1786243386243385,
        16/vertex//6.621536426536419/4.752844932844933,
        17/vertex//2.6411436711436647/5.039955229955231,
        18/vertex//10.471424501424492/5.222661782661783,
        19/leddar1//2.4192857142857083/7.010575905575906,
        20/leddar3//7.28711029711029/7.1149796499796505,
        21/leddar1//10.49752543752543/7.1149796499796505,
        22/leddar2//6.295274725274718/8.511379731379732,
        23/vertex//6.569334554334548/10.442849002849002
      } \node[\nodetype={\param}{}] (\nodename) at (\xpos, \ypos) {};
    \end{scope}
    \begin{scope} 
      \path (15) edge[conn,draw=leddar1color] (19);
      \path (6) edge[conn] (18);
      \path (5) edge[conn] (16);
      \path (17) edge[conn] (22);
      \path (4) edge[conn] (8);
      \path (19) edge[conn] (23);
      \path (18) edge[conn] (21);
      \path (14) edge[conn,draw=leddar1color] (21);
      \path (0) edge[conn] (1);
      \path (11) edge[conn] (18);
      \path (0) edge[conn] (3);
      \path (7) edge[conn] (13);
      \path (16) edge[conn] (22);
      \path (12) edge[conn] (14);
      \path (4) edge[conn] (7);
      \path (1) edge[conn] (16);
      \path (9) edge[conn] (20);
      \path (2) edge[conn] (10);
      \path (2) edge[conn] (8);
      \path (9) edge[conn] (15);
      \path (5) edge[conn] (12);
      \path (0) edge[conn] (2);
      \path (13) edge[conn] (19);
      \path (12) edge[conn] (20);
      \path (22) edge[conn] (23);
      \path (8) edge[conn] (15);
      \path (20) edge[conn] (23);
      \path (3) edge[conn] (11);
      \path (7) edge[conn] (14);
      \path (6) edge[conn] (13);
      \path (21) edge[conn] (23);
      \path (0) edge[conn] (5);
      \path (16) edge[conn] (20);
      \path (18) edge[conn] (22);
      \path (13) edge[conn] (21);
      \path (10) edge[conn,draw=leddar1color] (15);
      \path (17) edge[conn] (19);
      \path (5) edge[conn,draw=leddar1color] (11);
      \path (0) edge[conn] (4);
      \path (1) edge[conn] (9);
      \path (11) edge[conn,draw=leddar1color] (14);
      \path (1) edge[conn,draw=leddar1color] (10);
      \path (3) edge[conn] (6);
      \path (10) edge[conn] (17);
      \path (6) edge[conn] (17);
      \path (2) edge[conn] (6);
      \path (3) edge[conn] (7);
      \path (8) edge[conn] (13);
      \path (4) edge[conn] (9);
      \path (4) edge[conn] (12);
    \end{scope}
    \begin{scope} 
      \foreach \nodename/\labelpos/\labelopts/\labelcontent in {%
        1/above left/r/{C m, M ++,\ De -},
        2/above left/r/{C h, S u/n},
        3/above right/l/{S n, E --, M --},
        4/above/c/{R ++, Dl -, C vl, E ++},
        5/above right/l/{S o, R --},
        6/above//{M -},
        9/above//{S u},
        11/above//{E -},
        14/above//{C l},
        15/above//{E +},
        18/above//{De ++, Dl ++},
        19/above//{R +},
        20/above//{M +},
        21/above//{De +},
        22/above//{Dl +},
        1/below/o/{Conventional},
        2/below/o/{All-wheel},
        3/below/o/{Mid-engine},
        4/below/o/{Front-wheel},
        5/below/o/{Rear-wheel}
      } \coordinate[label={[\labelopts]\labelpos:{\labelcontent}}](c) at (\nodename);
    \end{scope}
  \end{scope}
\end{tikzpicture}

  \caption{Concept lattice (cf. Figure 1.14 in \textcite{fca-book}) for the
    context of drive concepts (Figure 1.13 \textcite{fca-book}). The
    extent ladders indicating the three interordinal scales are
    highlighted in color. The ordinal scaling dimension as well as
    order dimension of this context is four.}
  \label{fig:drive-concepts}
\end{figure}

To consolidate the understanding of the notions and statements on the
(interordinal) scaling dimension we provide an explanation based on a
small case example based on the \emph{drive concepts}~\cite{fca-book}
data set. This data set is a many-valued context consisting of five
objects, which characterize different ways of arranging the engine and
drive chain of a car, and seven many-valued attributes that measure
quality aspects for the driver, e.g., \emph{economy of space}. The
data set is accompanied by a scaling is that consists of a mixture of
bi-ordinal scalings of the quality (attribute) features, e.g.,
$\emph{good}<\emph{excellent}$ and $\emph{very poor}<\emph{poor}$, and
a nominal scaling for categorical features, e.g., for the
\emph{steering behavior}. The concept lattice of the scaled context
consists of twenty-four formal concepts and is depicted in
\cref{fig:drive-concepts}.

First we observe that the concept lattice of the example meets the
requirements to be derivable from interordinal scaling
(\cref{th55}). All objects are annotated to the atom concepts and the
complement of every attribute extent is an extent as well. The
interordinal scaling dimension of the scaled \emph{drive concept}
context is three which is much smaller then the original seven
many-valued attributes. Using the extent ladder characterization
provided in \cref{isd-bound} we highlighted three extent ladders in
color in the concept lattice diagram (see \cref{fig:drive-concepts}).
The first and largest extent ladder (highlighted in red) can be
inferred from the outer most concepts and covers sixteen out of
twenty-four concepts. The remaining two extent ladders have only two
elements and are of dichotomic scale.

\section{Discussion and Future Work}
The presented results on the scaling dimension have a number of
interfaces and correspondences to classical data science methods.  A
natural link to investigate would be comparing the scaling dimension
with standard correlation measures. Two features that correlate
prefectly, e.g., \cref{fig:interordinal-dimension}, induce an
equivalent conceptual scaling on the data. An analog of the scaling
dimension in this setting would be the smallest number of independent
features. Or, less strict, the smallest number of features such that
these features do not correlate more than some parameter. This obvious
similarity of both methods is breached by a key advantage of our
approach. In contrast to correlation measures, our method relies
solely on ordinal properties~\cite{measurement-levels} and does not
require the introduction of measurements for distance or ratios.

\cref{{odim-osd}} has already shown that there is a relationship between
an aspect of the scaling dimension of a formal context and the order
dimension of its concept lattice. The assumption that further such
relationships may exist is therefore reasonable. Yet, a thorough
investigation of these relationships is an extensive research program
in its own right and therefore cannot be addressed within the scope of
this paper. An investigation on how the scaling dimension relates to
other measures of dimension within the realm of
FCA~\cite{IntrinsicDimension,Tatti2006WhatIT} is therefore deemed
future work.

Due to novel insights into the computational tractability of
recognizing scale-measures~\cite{smeasure} (that is in preparation and
will be made public later this year) we have little hope that the
scaling dimension and interordinal scaling dimension can be decided in
polynomial time. Despite that, efficient algorithms for real-world
data that compute the scaling dimension and its specific versions,
i.e., ordinal, interordinal, nominal, etc, may be developed.  In
addition to that, so far it is unknown if an approximation of the
scaling dimension, e.g., with respect to some degree of conceptual
scaling error \cite{smeasure-error} or bounds, is tractable. If
computationally feasible, such an approximation could allow larger
data sets to be handled.

Another line of research that can be pursued in future work is how the
scaling dimension can be utilized to derive more readable line
diagrams. We can envision that diagrams of concept lattices that are
composed of fewer scales, i.e., have a lower scaling dimension, are
more readable even if they have slightly more concepts. An open
problem that needs to be solved here is: for a context $\context$ and
$k\in \mathbb{N}$ the identification of $k$ scales that cover the
largest number of concepts from $\BV(\context)$ with respect to scale
measures.

\section{Conclusion}
With our work, we contributed towards a deeper understanding of
conceptual scaling~\cite{scaling}. In particular, we introduced the
notion of pre-scaling to formalize background knowledge on attribute
domains, e.g., underlying order relations, that can be used and
extended to scales in the process of conceptual scaling.
To deal with the complexity of scalings selection or logical
compression methods have been proposed to reflect parts of the
conceptual structure~\cite{smeasure}. Furthermore, we introduced the
notions of conceptual and contextual views to characterize these
methods and provided a first formal definition.

We extended the realm of conceptual measurability~\cite{cmeasure} by
the scaling dimension, i.e., the least number of attributes needed to
derive a context by the means of plain scaling. This notion does not
only provide insight towards the complexity of an underlying scaling
but can also be applied for many-valued feature compression. For the
identification of the scaling dimension, we provided characterizations
for the ordinal and interordinal scaling dimension in terms of
structural properties of the concept lattice. These employ chains of
meet-irreducible extents and newly introduced extent ladders. We
demonstrated their applicability based on the drive concepts data set
and highlighted the identified extent ladders and chains in the
concept lattice diagram. Our analysis showed that while the
many-valued context consists of seven many-valued attributes an
equivalent scaling can be derived from three interordinally scaled or
four ordinal scaled many-valued attributes. 

In addition to the structural characterizations of the scaling
dimensions, we provided bounds for the interordinal and ordinal
scaling dimension. In detail, we showed upper and lower bounds in terms
of the width and the order dimension of the concept lattice. This
result shows in particular how far-reaching and therefore necessary a
future in-depth investigation of the scaling dimension is.

\printbibliography
\end{document}